%% file: icml13_listopt_arxiv.tex
\documentclass{article}
\usepackage[disable]{todonotes}
\usepackage{graphicx} 
\usepackage{subfigure}
\usepackage{amssymb}

\usepackage{natbib}

\usepackage{algorithm}
\usepackage{algorithmic}

\usepackage{hyperref}

\usepackage[accepted]{icml2013} 

\icmltitlerunning{Learning Policies for Contextual Submodular Prediction}

\usepackage{amsthm,amsmath}
\newtheorem{lemma}{Lemma}
\newtheorem{corollary}{Corollary}
\newtheorem{theorem}{Theorem}

\DeclareMathOperator*{\argmax}{argmax}

\newenvironment{itemize*}
  {\begin{itemize}
   \setlength{\topsep}{0pt}
    \setlength{\itemsep}{0pt}
    \setlength{\parskip}{0pt}}
  {\end{itemize}}

\begin{document}

\twocolumn[
\icmltitle{Learning Policies for Contextual Submodular Prediction}

\icmlauthor{Stephane Ross}{stephaneross@cmu.edu}
\icmlauthor{Jiaji Zhou}{jiajiz@andrew.cmu.edu}
\icmlauthor{Yisong Yue}{yisongyue@cmu.edu}
\icmlauthor{Debadeepta Dey}{debadeep@cs.cmu.edu}
\icmlauthor{J. Andrew Bagnell}{dbagnell@ri.cmu.edu}
\icmladdress{School of Computer Science, Carnegie Mellon University, Pittsburgh, PA, USA}

\icmlkeywords{List Optimization, Submodularity, Online Learning}

\vskip 0.3in
]
\begin{abstract}

\input{abstractBagnell}
\end{abstract}





\section{Introduction}
\input{introduction}
\input{related}

\section{Background}

\input{background}
\input{ContextFree}

\section{Contextual List Optimization with Stationary Policies}
\label{sec:contextual}

\input{Contextual.tex}


\section{Experimental Results}
\label{sec:experiments}

\begin{figure*}
\centering
\subfigure[]{\includegraphics[width=0.3\textwidth]{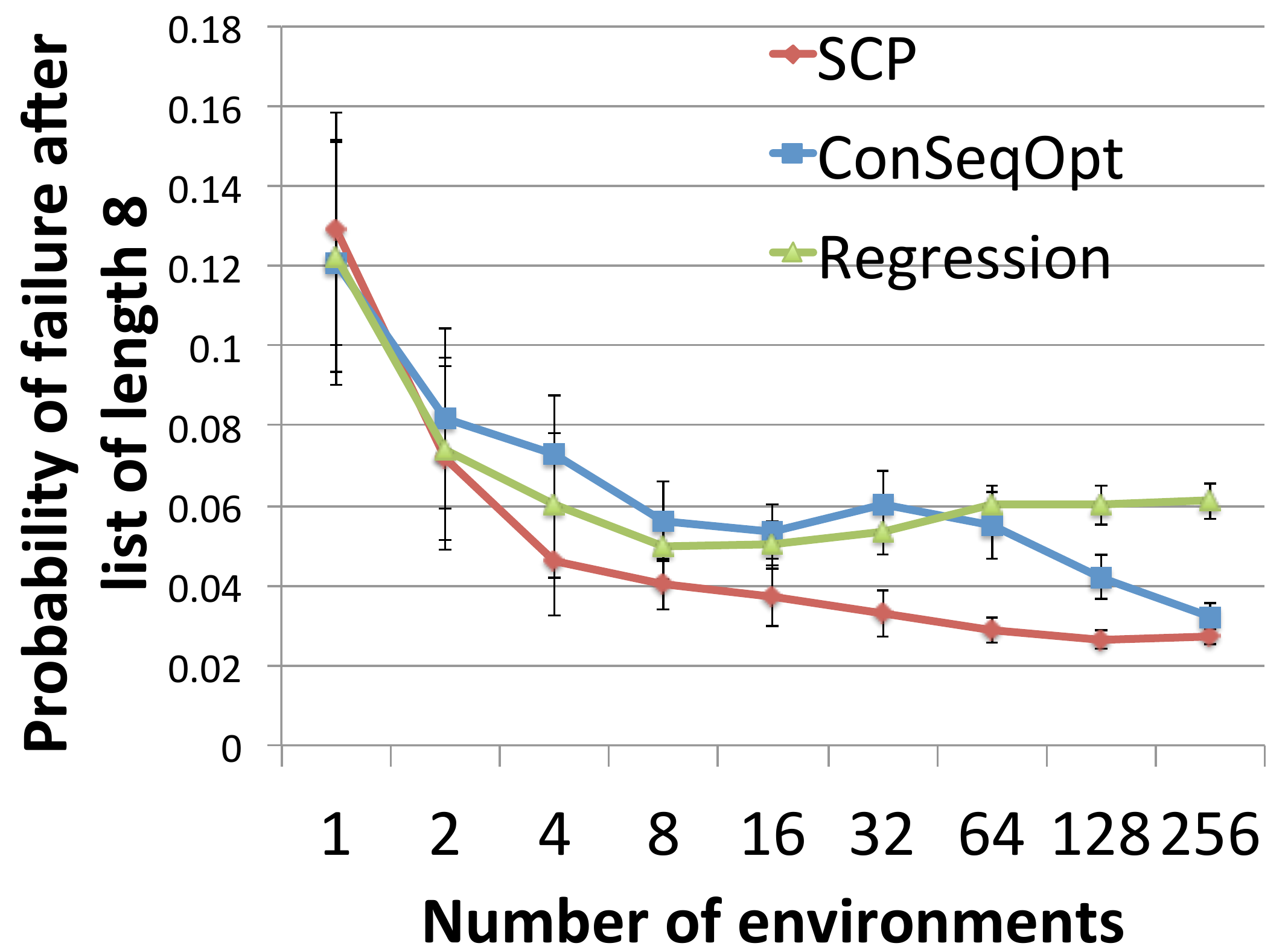}}\quad
\subfigure[]{\includegraphics[width=0.3\textwidth]{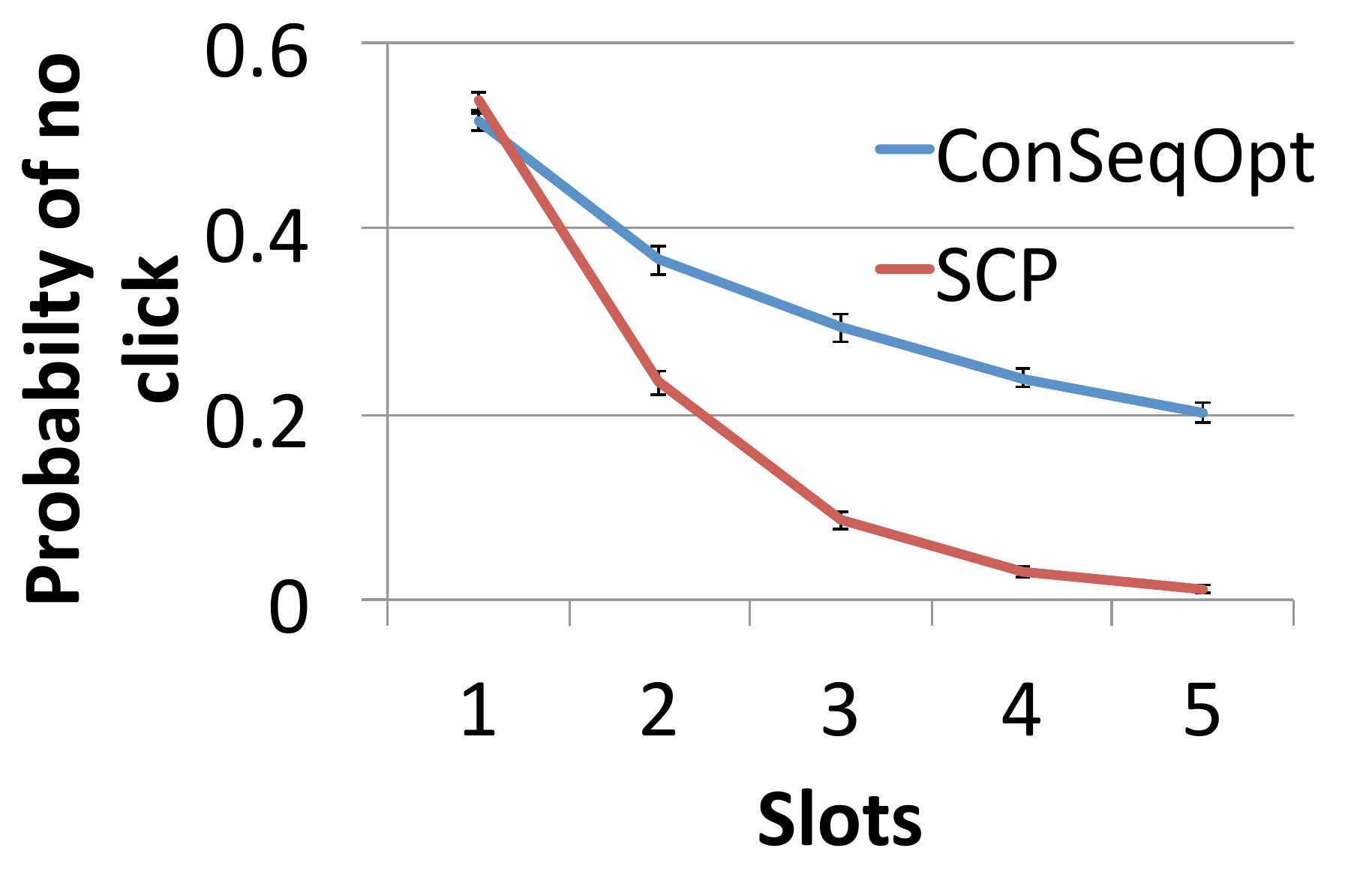}}\quad
\subfigure[]{\includegraphics[width=0.3\textwidth]{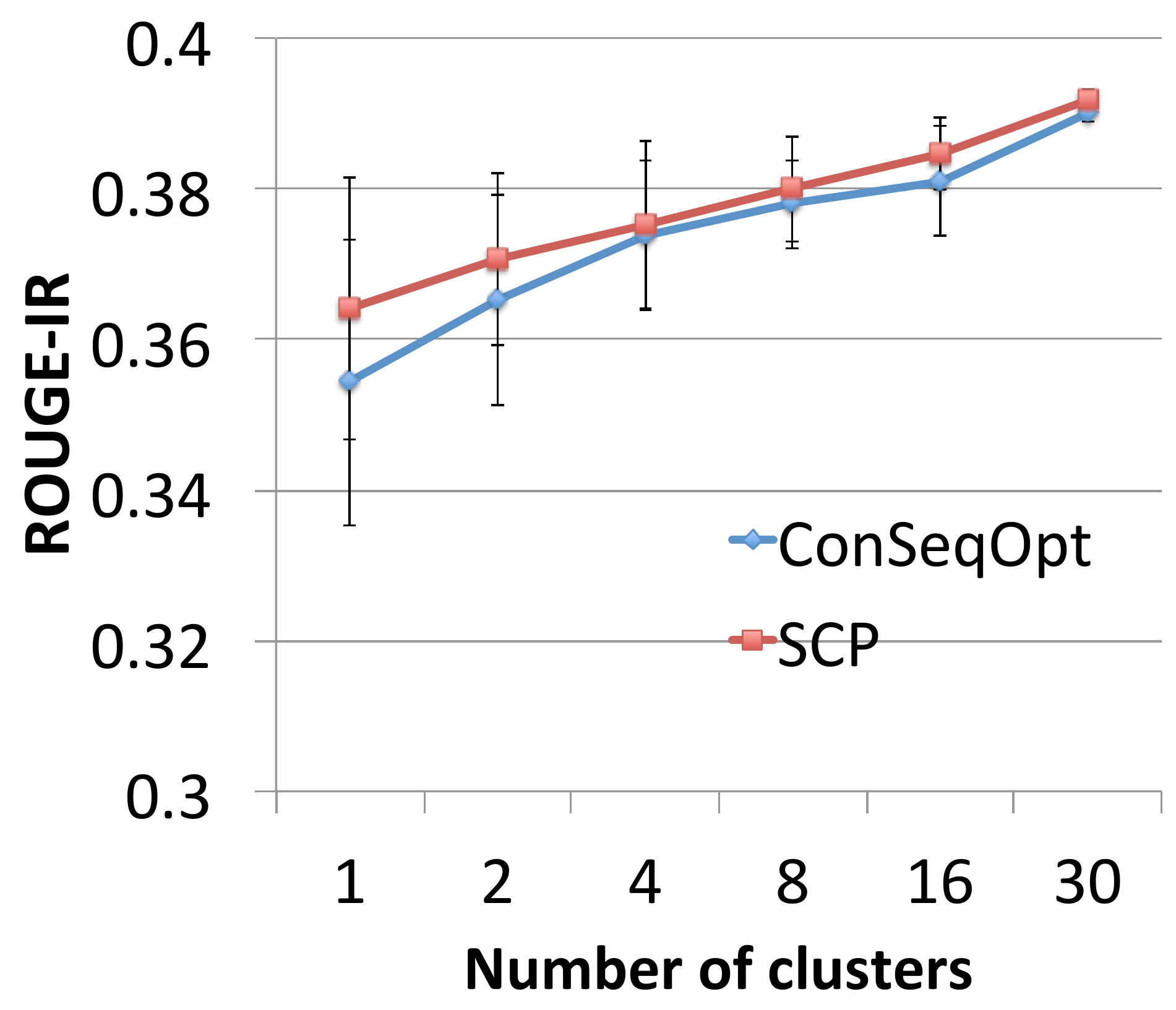}}\quad
\vspace{-0.1in}
\caption{(a)SCP performs better at even low data availability while ConSeqOpt
suffers from data starvation issues (b) With increase in slots SCP predicts news articles which have lower
probability of the user not clicking on any of them compared to ConSeqOpt (c) ROUGE-1R scores with respect to the size of the training data}
\label{all_figs}
\end{figure*}


\subsection{Robotic Manipulation Planning}
\input{traj_opt_sec}
\subsection{Personalized News Recommendation}
\input{news}
\subsection{Document Summarization}
\input{docsum}

\input{acknowledgement}



\clearpage

\appendix

\newtheorem*{thmSCPDef}{Theorem \ref{thmSCP}}
\newtheorem*{corWMDef}{Corollary \ref{corWM}}
\newtheorem*{corSCPDef}{Corollary \ref{corSCP}}

\section{Proofs of Theoretical Results}

This appendix contains the proofs of the various theoretical results presented in this paper. 

\subsection{Preliminaries}

We begin by proving a number of lemmas about monotone submodular functions, which will be useful to prove our main results.

\begin{lemma} \label{lemAddList}
Let $\mathcal{S}$ be a set and $f$ be a monotone submodular function defined on list of items from $\mathcal{S}$. For any lists $A,B$, we have that:
\begin{displaymath}
f(A \oplus B) - f(A) \leq |B| ( \mathbb{E}_{s \sim U(B)}[f(A \oplus s)] - f(A) )
\end{displaymath}
for $U(B)$ the uniform distribution on items in $B$. 
\end{lemma}
\begin{proof}
For any list $A$ and $B$, let $B_i$ denote the list of the first $i$ items in $B$, and $b_i$ the $i^{th}$ item in $B$. We have that:
\begin{displaymath}
\begin{array}{rl}
\multicolumn{2}{l}{f(A \oplus B) - f(A)}\\
= & \sum_{i=1}^{|B|} f(A \oplus B_i) -f(A \oplus B_{i-1})\\
\leq & \sum_{i=1}^{|B|} f(A \oplus b_i) -f(A) \\
= & |B| ( \mathbb{E}_{b \sim U(B)}[f(A \oplus b)] - f(A) )\\
\end{array}
\end{displaymath}
where the inequality follows from the submodularity property of $f$.
\end{proof}

\begin{lemma} \label{lemBudgetError}
Let $\mathcal{S}$ be a set, and $f$ a monotone submodular function defined on lists of items in $\mathcal{S}$. Let $A,B$ be any lists of items from $\mathcal{S}$. Denote $A_j$ the list of the first $j$ items in $A$, $U(B)$ the uniform distribution on items in $B$ and define $\epsilon_j = \mathbb{E}_{s \sim U(B)}[f(A_{j-1} \oplus s)] - f(A_{j})$, the additive error term in competing with the average marginal benefits of the items in $B$ when picking the $j^{th}$ item in $A$ (which could be positive or negative). Then:
\begin{displaymath}
f(A) \geq (1-(1-1/|B|)^{|A|}) f(B) - \sum_{i=1}^{|A|} (1-1/|B|)^{|A|-i} \epsilon_i
\end{displaymath}
In particular if $|A| = |B| = k$, then:
\begin{displaymath}
f(A) \geq (1-1/e) f(B) - \sum_{i=1}^k (1-1/k)^{k-i} \epsilon_i
\end{displaymath}
and for $\alpha = \exp(-|A|/|B|)$ (i.e. $|A| = |B| \log(1/\alpha)$):
\begin{displaymath}
f(A) \geq (1-\alpha) f(B) - \sum_{i=1}^{|A|} (1-1/|B|)^{|A|-i} \epsilon_i
\end{displaymath}
\end{lemma}
\begin{proof}
Using the monotone property and previous lemma \ref{lemAddList}, we must have that: $f(B)-f(A) \leq f(A \oplus B) - f(A) \leq  |B| ( \mathbb{E}_{b \sim U(B)}[f(A \oplus b)] - f(A) )$.

Now let $\Delta_j = f(B) - f(A_j)$. By the above we have that
\begin{displaymath}
\begin{array}{rl}
\multicolumn{2}{l}{\Delta_j}\\ 
\leq & |B| [\mathbb{E}_{s \sim U(B)} [f(A_j \oplus s)] - f(A_j)] \\
= & |B| [\mathbb{E}_{s \sim U(B)} [f(A_j \oplus s)] - f(A_{j+1}) \\
& + f(A_{j+1}) - f(B) + f(B) - f(A_j)] \\
= &  |B| [\epsilon_{j+1} + \Delta_j - \Delta_{j+1} ] \\
\end{array}
\end{displaymath}

Rearranging terms, this implies that $\Delta_{j+1} \leq (1-1/|B|) \Delta_j + \epsilon_{j+1}$. Recursively expanding this recurrence from $\Delta_{|A|}$, we obtain:
\begin{displaymath}
\Delta_{|A|} \leq (1-1/|B|)^{|A|} \Delta_0 + \sum_{i = 1}^{|A|} (1-1/|B|)^{|A|-i} \epsilon_i
\end{displaymath}
Using the definition of $\Delta_{|A|}$ and rearranging terms,  we obtain $f(A) \geq (1-(1-1/|B|)^{|A|}) f(B) - \sum_{i=1}^{|A|} (1-1/|B|)^{|A|-i} \epsilon_i$. This proves the first statement of the theorem. The following two statements follow from the observations that $(1-1/|B|)^{|A|} = \exp(|A|\log(1-1/|B|)) \leq \exp(-|A|/|B|) = \alpha$. Hence $(1-(1-1/|B|)^{|A|}) f(B) \geq (1-\alpha) f(B)$. When $|A|=|B|$, $\alpha = 1/e$ and this proves the special case where $|A| = |B|$.
\end{proof}

For the greedy list construction strategy, the $\epsilon_j$ in the last lemma are always $\leq 0$, such that Lemma \ref{lemBudgetError} implies that if we construct a list of size $k$ with greedy, it must achieve at least ~63\% of the value of the optimal list of size $k$, but also that it must achieve at least ~95\% of the value of the optimal list of size $\lfloor k/3 \rfloor$, and at least ~99.9\% of the value of the optimal list of size $\lfloor k/7 \rfloor$.

A more surprising fact that follows from the last lemma is that constructing a list stochastically, by sampling items from a particular fixed distribution, can provide the same guarantee as greedy: 

\begin{lemma} \label{lemStochasticRatio}
Let $\mathcal{S}$ be a set, and $f$ a monotone submodular function defined on lists of items in $\mathcal{S}$. Let $B$ be any list of items from $\mathcal{S}$ and $U(B)$ the uniform distribution on elements in $B$. Suppose we construct the list $A$ by sampling $k$ items randomly from $U(B)$ (with replacement). Denote $A_j$ the list obtained after $j$ samples, and $P_j$ the distribution over lists obtained after $j$ samples. Then:
\begin{displaymath}
\mathbb{E}_{A \sim P_k}[f(A)] \geq (1-(1-1/|B|)^k) f(B)
\end{displaymath}
In particular, for $\alpha = \exp(-k/|B|)$:
\begin{displaymath}
\mathbb{E}_{A \sim P_k}[f(A)] \geq (1-\alpha) f(B)
\end{displaymath}
\end{lemma}
\begin{proof}
The proof follows a similar proof to the previous lemma. Recall that by the monotone property and lemma \ref{lemAddList}, we have that for any list $A$: $f(B)-f(A) \leq f(A \oplus B) - f(A) \leq  |B| ( \mathbb{E}_{b \sim U(B)}[f(A \oplus b)] - f(A) )$. Because this holds for all lists, we must also have that for any distribution $P$ over lists $A$, $f(B) - \mathbb{E}_{A \sim P}[f(A)] \leq |B| \mathbb{E}_{A \sim P}[\mathbb{E}_{b \sim U(B)}[f(A \oplus b)] - f(A)]$. Also note that by the way we construct sets, we have that $\mathbb{E}_{A_{j+1} \sim P_{j+1}}[f(A_{j+1})] = \mathbb{E}_{A_j \sim P_j}[\mathbb{E}_{s \sim U(B)} [f(A_j \oplus s)]]$

Now let $\Delta_j = f(B) - \mathbb{E}_{A_j \sim P_j}[f(A_j)]$. By the above we have that:
\begin{displaymath}
\begin{array}{rl}
\multicolumn{2}{l}{\Delta_j}\\ 
\leq & |B| \mathbb{E}_{A_j \sim P_j}[\mathbb{E}_{s \sim U(B)} [f(A_j \oplus s)] - f(A_j)] \\
= & |B| \mathbb{E}_{A_j \sim P_j}[\mathbb{E}_{s \sim U(B)} [f(A_j \oplus s)]  - f(B) \\
& + f(B) - f(A_j)] \\
= & |B|( \mathbb{E}_{A_{j+1} \sim P_{j+1}}[f(A_{j+1)}]  - f(B) \\
& + f(B) - \mathbb{E}_{A_j \sim P_j}[f(A_j)] )\\
= &  |B| [\Delta_j - \Delta_{j+1} ] \\
\end{array}
\end{displaymath}
Rearranging terms, this implies that $\Delta_{j+1} \leq (1-1/|B|) \Delta_j$. Recursively expanding this recurrence from $\Delta_k$, we obtain:
\begin{displaymath}
\Delta_{k} \leq (1-1/|B|)^k \Delta_0
\end{displaymath}
Using the definition of $\Delta_k$ and rearranging terms we obtain $E_{A \sim P_k}[f(A)] \geq (1-(1-1/|B|)^k) f(B)$. The second statement follows again from the fact that $(1-(1-1/|B|)^k) f(B) \geq (1-\alpha) f(B)$
\end{proof}

\begin{corollary}
There exists a distribution that when sampled $k$ times to construct a list, achieves an approximation ratio of $(1-1/e)$ of the optimal list of size $k$ in expectation. In particular, if $A^*$ is an optimal list of size $k$, sampling $k$ times from $U(A^*)$ achieves this approximation ratio. Additionally, for any $\alpha \in (0,1]$, sampling $\lceil k\log(1/\alpha) \rceil$ times must construct a list that achieves an approximation ratio of $(1-\alpha)$ in expectation.
\end{corollary}
\begin{proof}
Follows from the last lemma using $B=A^*$.
\end{proof}

This surprising result can also be seen as a special case of a more general result proven in prior related work that analyzed randomized set selection strategies to optimize submodular functions (lemma 2.2 in \cite{feige2011}). 

\subsection{Proofs of Main Results}

We now provide the proofs of the main results in this paper. We provide the proofs for the more general contextual case where we learn over a policy class $\tilde{\Pi}$. All the results for the context-free case can be seen as special cases of these results when $\Pi = \tilde{\Pi} = \{ \pi_s | s \in \mathcal{S} \}$ and $\pi_s(x,L) = s$ for any state $x$ and list $L$.

We refer the reader to the notation defined in section \ref{sec:background} and \ref{sec:contextual} for the definitions of the various terms used.

\begin{thmSCPDef}
Let $\alpha = \exp(-m/k)$ and $k' = \min(m,k)$. After $T$ iterations, for any $\delta,\delta' \in (0,1)$, we have that with probability at least $1-\delta$:
\begin{displaymath}
F(\overline{\pi},m) \geq (1-\alpha) F(L^*_{\pi,k}) - \frac{R}{T} - 2 \sqrt{\frac{2 \ln(1/\delta)}{T}}
\end{displaymath}
and similarly, with probability at least $1-\delta-\delta'$:
\begin{displaymath}
\begin{array}{rl}
F(\overline{\pi},m) \geq & (1-\alpha) F(L^*_{\pi,k}) - \frac{\mathbb{E}[R]}{T}  - \sqrt{\frac{2k'\ln(1/\delta')}{T}} \\
& - 2 \sqrt{\frac{2 \ln(1/\delta)}{T}}
\end{array}
\end{displaymath}
\end{thmSCPDef}
\begin{proof}
\begin{displaymath}
\begin{array}{rl}
\multicolumn{2}{l}{F(\overline{\pi},m)} \\
= & \frac{1}{T} \sum_{t=1}^T F(\pi_t, m)\\
=  & \frac{1}{T}\sum_{t=1}^T \mathbb{E}_{L_{\pi,m} \sim \pi_t} [\mathbb{E}_{x \sim D} [ f_x(L_{\pi,m}(x)) ] ]\\
= & (1-\alpha)\mathbb{E}_{x \sim D}[f_x(L^*_{\pi,k}(x))] \\
& - [(1-\alpha)\mathbb{E}_{x \sim D}[f_x(L^*_{\pi,k}(x))]  \\
& - \frac{1}{T}\sum_{t=1}^T \mathbb{E}_{L_{\pi,m} \sim \pi_t} [\mathbb{E}_{x \sim D} [ f_x(L_{\pi,m}(x)) ] ] ] \\
\end{array}
\end{displaymath}
Now consider the sampled states $\{x_t\}_{t=1}^T$ and the policies $\pi_{t,i}$ sampled i.i.d. from $\pi_t$ to construct the lists $\{L_t\}_{t=1}^T$ and denote the random variables $X_t = (1-\alpha) ( \mathbb{E}_{x \sim D}[f_x(L^*_{\pi,k}(x))] - f_{x_t}(L^*_{\pi,k}(x_t)) ) - \mathbb{E}_{L_{\pi,m} \sim \pi_t} [\mathbb{E}_{x \sim D} [ f_x(L_{\pi,m}(x)) ] ] - f_{x_t}(L_t) ]$. If $\pi_t$ is deterministic, then simply consider all $\pi_{t,i} = \pi_t$. Because the $x_t$ are i.i.d. from $D$, and the distribution of policies used to construct $L_t$ only depends on $\{ x_\tau \}_{\tau=1}^{t-1}$ and $\{ L_{\tau} \}_{\tau=1}^{t-1}$, then the $X_t$ conditioned on $\{ X_{\tau} \}_{\tau=1}^{t-1}$ have expectation 0, and because $f_x \in [0,1]$ for all state $x \in \mathcal{X}$, $X_t$ can vary in a range $r \subseteq [-2,2]$. Thus the sequence of random variables $Y_t = \sum_{i=1}^t X_i$, for t =$1$ to $T$, forms a martingale where $|Y_t - Y_{t+1}| \leq 2$. By the Azuma-Hoeffding's inequality, we have that $P(Y_T/T \geq \epsilon) \leq \exp(-\epsilon^2 T/8)$. Hence for any $\delta \in (0,1)$, we have that with probability at least $1-\delta$,  $Y_T/T \leq 2 \sqrt{\frac{2 \ln(1/\delta)}{T}}$. Hence we have that with probability at least $1-\delta$:
\begin{displaymath}
\begin{array}{rl}
\multicolumn{2}{l}{F(\overline{\pi},m)} \\
= & (1-\alpha)\mathbb{E}_{x \sim D}[f_x(L^*_{\pi,k}(x))] \\
& - [(1-\alpha)\mathbb{E}_{x \sim D}[f_x(L^*_{\pi,k}(x))] \\
& - \frac{1}{T}\sum_{t=1}^T \mathbb{E}_{L_{\pi,m} \sim \pi_t} [\mathbb{E}_{x \sim D} [ f_x(L_{\pi,m}(x)) ] ] ] \\
=  & (1-\alpha)\mathbb{E}_{x \sim D}[f_x(L^*_{\pi,k}(x))] \\
& - [(1-\alpha)\frac{1}{T} \sum_{t=1}^T f_{x_t}(L^*_{\pi,k}(x_t)) \\
&  - \frac{1}{T}\sum_{t=1}^T f_{x_t}(L_t) ] - Y_T/T\\
=  & (1-\alpha)\mathbb{E}_{x \sim D}[f_x(L^*_{\pi,k}(x))] \\
& - [(1-\alpha)\frac{1}{T} \sum_{t=1}^T f_{x_t}(L^*_{\pi,k}(x_t)) \\
& - \frac{1}{T}\sum_{t=1}^T f_{x_t}(L_t) ] - 2 \sqrt{\frac{2 \ln(1/\delta)}{T}}\\
\end{array}
\end{displaymath}

Let $w_i = (1-1/k)^{m-i}$. From Lemma \ref{lemBudgetError}, we have:
\begin{displaymath}
\begin{array}{rl}
\multicolumn{2}{l}{(1-\alpha)\frac{1}{T} \sum_{t=1}^T f_{x_t}(L^*_{\pi,k}(x_t))  - \frac{1}{T}\sum_{t=1}^T f_{x_t}(L_t)}\\
\leq & \frac{1}{T}\sum_{t=1}^T\sum_{i=1}^m w_i (\mathbb{E}_{\pi \sim U(L^*_{\pi,k})}[f_{x_t}(L_{t,i-1} \oplus \pi(x_t))] \\
& - f_{x_t}(L_{t,i}))\\
= & \mathbb{E}_{\pi \sim U(L^*_{\pi,k})}[\frac{1}{T}\sum_{t=1}^T\sum_{i=1}^m w_i (f_{x_t}(L_{t,i-1} \oplus \pi(x_t)) \\
& - f_{x_t}(L_{t,i}))]\\
\leq & \max_{\pi \in \Pi}[\frac{1}{T}\sum_{t=1}^T\sum_{i=1}^m w_i (f_{x_t}(L_{t,i-1} \oplus \pi(x_t)) \\
& - f_{x_t}(L_{t,i}))] \\
\leq & \max_{\pi \in \tilde{\Pi}}[\frac{1}{T}\sum_{t=1}^T\sum_{i=1}^m w_i (f(L_{t,i-1} \oplus \pi(x_t)) \\
& - f_{x_t}(L_{t,i}))] \\
= & R/T\\
\end{array}
\end{displaymath}
Hence combining with the previous result proves the first part of the theorem.

Additionally, for the sampled environments $\{x_t\}_{t=1}^T$ and the policies $\pi_{t,i}$, consider the random variables $Q_{m(t-1)+i} = w_i \mathbb{E}_{\pi \sim \pi_t}[ f_{x_t}(L_{t,i-1} \oplus \pi(x_t,L_{t,i-1})) ] - w_i f_{x_t}(L_{t,i})$. Because each draw of $\pi_{t,i}$ is i.i.d. from $\pi_t$, we have that again the sequence of random variables $Z_j = \sum_{i=1}^j Q_i$, for $j = 1$ to $Tm$ forms a martingale and because each $Q_i$ can take values in a range $[-w_j,w_j]$ for $j = 1 + \textrm{mod}(i-1,m)$, we have $|Z_i - Z_{i-1}| \leq w_j$. Since $\sum_{i=1}^{Tm} |Z_i - Z_{i-1}|^2 \leq T \sum_{i=1}^m  (1-1/k)^{2(m-i)} \leq T\min(k,m) = Tk'$,  by Azuma-Hoeffding's inequality, we must have that $P(Z_{Tm} \geq \epsilon) \leq \exp(-\epsilon^2/2Tk')$. Thus for any $\delta' \in (0,1)$, with probability at least $1-\delta'$, $Z_{Tm} \leq \sqrt{2Tk'\ln(1/\delta)}$. Hence combining with the previous result, it must be the case that with probability at least $1-\delta-\delta'$, both $Y_T/T \leq 2 \sqrt{\frac{2 \ln(1/\delta)}{T}}$ and $Z_{Tm} \leq \sqrt{2Tk'\ln(1/\delta')}$ holds.

Now note that:
\begin{displaymath}
\begin{array}{rl}
\multicolumn{2}{l}{ \max_{\pi \in \tilde{\Pi}}[\frac{1}{T}\sum_{t=1}^T\sum_{i=1}^m w_i (f(L_{t,i-1} \oplus \pi(x_t)) - f_{x_t}(L_{t,i}))] } \\
= & \max_{\pi \in \tilde{\Pi}}[\frac{1}{T}\sum_{t=1}^T\sum_{i=1}^m w_i (f_{x_t}(L_{t,i-1} \oplus \pi(x_t)) \\
& - \mathbb{E}_{\pi' \sim \pi_t}[f(L_{t,i-1} \oplus \pi'(x_t,L_{t,i-1}))])] + Z_{Tm}/T \\
= & \mathbb{E}[R]/T + Z_{Tm}/T \\
\end{array}
\end{displaymath}

Using this additional fact, and combining with previous results we must have that with probability at least $1-\delta-\delta'$:  
\begin{displaymath}
\begin{array}{rl}
\multicolumn{2}{l}{F(\overline{\pi},m)} \\
\geq &  (1-\alpha)F(L^*_{\pi,k}) - [(1-\alpha)\frac{1}{T} \sum_{t=1}^T f_{x_t}(L^*_{\pi,k}(x_t)) \\
&  - \frac{1}{T}\sum_{t=1}^T  f_{x_t}(L_t) ] - 2 \sqrt{\frac{2 \ln(1/\delta)}{T}}\\
\geq & (1-\alpha)F(L^*_{\pi,k}) - \mathbb{E}[R]/T - Z_{Tm}/T - 2 \sqrt{\frac{2 \ln(1/\delta)}{T}} \\
\geq & (1-\alpha)F(L^*_{\pi,k})  -  \mathbb{E}[R]/T - \sqrt{\frac{2k'\ln(1/\delta')}{T}}\\
& - 2 \sqrt{\frac{2 \ln(1/\delta)}{T}} \\
\end{array}
\end{displaymath}
\end{proof}

We now show that the expected regret must grow with $\sqrt{k'}$ and not $k'$, hen using Weighted Majority with the optimal learning rate (or with the doubling trick).
\begin{corWMDef}
Under the event where Theorem \ref{thmSCP} holds (the event that occurs w.p. $1-\delta-\delta'$), if $\tilde{\Pi}$ is a finite set of policies, using Weighted Majority with the optimal learning rate guarantees that after $T$ iterations:
\begin{displaymath}
\begin{array}{rcl}
\mathbb{E}[R]/T & \leq & \frac{4k' \ln|\tilde\Pi|}{T} + 2\sqrt{\frac{k' \ln|\tilde\Pi|}{T}} \\
& & + 2^{9/4}(k'/T)^{3/4}(\ln(1/\delta'))^{1/4} \sqrt{\ln|\tilde\Pi|}
\end{array}
\end{displaymath}
For large enough $T$ in $\Omega(k' (\ln|\tilde\Pi| + \ln(1/\delta')))$, we obtain that:\begin{displaymath}
\mathbb{E}[R]/T \leq O( \sqrt{\frac{k'\ln|\tilde\Pi|}{T}} )
\end{displaymath}
\end{corWMDef}
\begin{proof}
We use a similar argument to Streeter \& Golovin Lemma 4 \cite{Streeter07tech} to bound $\mathbb{E}[R]$ in the result of theorem \ref{thmSCP}. Consider the sum of the benefits accumulated by the learning algorithm at position $i$ in the list, for $i \in {1,2,\dots,m}$, i.e. let $y_i = \sum_{t=1}^T b(\pi_{t,i}(x_t,L_{t,i-1})|x_t,L_{t,i-1})$, where $\pi_{t,i}$ corresponds to the particular sampled policy by Weighted Majority for choosing the item at position $i$, when constructing the list $L_t$ for state $x_t$. Note that $\sum_{i=1}^m (1-1/k)^{m-i} y_i \leq \sum_{i=1}^m y_i \leq T$ by the fact that the monotone submodular function $f_x$ is bounded in $[0,1]$ for all state $x$. Now consider the sum of the benefits you could have accumulated at position $i$, had you chosen the best fixed policy in hindsight to construct all list, keeping the policy fixed as the policy is constructed, i.e. let $z_i = \sum_{t=1}^T b(\pi^*(x_t,L_{t,i-1})|x_t,L_{t,i-1})$, for $\pi^* = \arg\max_{\pi \in \tilde{\Pi}} \sum_{i=1}^m (1-1/k)^{m-i}\sum_{t=1}^T b(\pi^*(x_t,L_{t,i-1})|x_t,L_{t,i-1})$ and let  $r_i = z_i - y_i$. Now denote $Z = \sqrt{\sum_{i=1}^m (1-1/k)^{m-i} z_i}$. We have $Z^2 = \sum_{i=1}^m (1-1/k)^{m-i} z_i = \sum_{i=1}^m (1-1/k)^{m-i} (y_i + r_i) \leq T + R$, where $R$ is the sample regret incurred by the learning algorithm. Under the event where theorem \ref{thmSCP} holds (i.e. the event that occurs with probability at least $1-\delta-\delta'$), we had already shown that $R \leq \mathbb{E}[R] + Z_{Tm}$, for $Z_{Tm} \leq \sqrt{2Tk'\ln(1/\delta')}$, in the second part of the proof of theorem \ref{thmSCP}. Thus when theorem \ref{thmSCP} holds, we have $Z^2 \leq T + \sqrt{2Tk'\ln(1/\delta')} + \mathbb{E}[R]$. Now using the generalized version of weighted majority with rewards (i.e. using directly the benefits as rewards) \cite{Arora12}, since the rewards at each update are in $[0,k']$, we have that with the best learning rate in hindsight \footnote{if not a doubling trick can be used to get the same regret bound within a small constant factor \cite{CesaBianchi97}}: $\mathbb{E}[R] \leq 2Z\sqrt{k' \ln|\tilde\Pi|}$. Thus we obtain $Z^2 \leq T + \sqrt{2Tk'\ln(1/\delta')} + 2Z\sqrt{k' \ln|\tilde\Pi|}$. This is a quadratic inequality of the form $Z^2 - 2Z\sqrt{k' \ln|\tilde\Pi|} - T - \sqrt{2Tk'\ln(1/\delta')} \leq 0$, with the additional constraint $Z \geq 0$. This implies $Z$ is less than or equal to the largest non-negative root of the polynomial $Z^2 - 2Z\sqrt{k' \ln|\tilde\Pi|} - T - \sqrt{2Tk'\ln(1/\delta')}$. Solving for the roots, we obtain
 \begin{displaymath}
 \begin{array}{rcl}
 Z & \leq & \sqrt{k' \ln|\tilde\Pi|}  + \sqrt{k' \ln|\tilde\Pi| + T + \sqrt{2Tk'\ln(1/\delta')}} \\
 & \leq & 2 \sqrt{k' \ln|\tilde\Pi|} + \sqrt{T} + (2Tk'\ln(1/\delta'))^{1/4} \\
 \end{array}
\end{displaymath} 
Plugging back $Z$ into the expression $\mathbb{E}[R] \leq 2Z\sqrt{k' \ln|\tilde\Pi|}$, we obtain:
\begin{displaymath}
\begin{array}{rl}
\mathbb{E}[R] \leq & 4k' \ln|\tilde\Pi| + 2\sqrt{Tk' \ln|\tilde\Pi|} \\
&+ 2(2T\ln(1/\delta'))^{1/4}(k')^{3/4}\sqrt{\ln|\tilde\Pi|} 
\end{array}
\end{displaymath}
Thus the average regret:
 \begin{displaymath}
 \begin{array}{rl}
 \frac{\mathbb{E}[R]}{T} \leq& \frac{4k' \ln|\tilde\Pi|}{T} + 2\sqrt{\frac{k' \ln|\tilde\Pi|}{T}} \\
 & + 2^{9/4}(k'/T)^{3/4}(\ln(1/\delta'))^{1/4} \sqrt{\ln|\tilde\Pi|}
 \end{array}
 \end{displaymath}
 For $T$ in $\Omega(k' (\ln \tilde{\Pi} + \ln(1/\delta')))$, the dominant term is $2\sqrt{\frac{k' \ln|\tilde\Pi|}{T}}$, and thus $\frac{\mathbb{E}[R]}{T}$ is $O(\sqrt{\frac{k' \ln|\tilde\Pi|}{T}})$. 
\end{proof}

\begin{corSCPDef}
Let $\alpha = \exp(-m/k)$ and $k' = \min(m,k)$. If we run an online learning algorithm on the sequence of convex loss $C_t$ instead of $\ell_t$, then after $T$ iterations, for any $\delta \in (0,1)$, we have that with probability at least $1-\delta$:
\begin{displaymath}
F(\overline{\pi},m) \geq (1-\alpha) F(L^*_{\pi,k}) - \frac{\tilde{R}}{T} - 2 \sqrt{\frac{2 \ln(1/\delta)}{T}} - \mathcal{G}
\end{displaymath}
where $\tilde{R}$ is the regret on the sequence of convex loss $C_t$, and $\mathcal{G} = \frac{1}{T}[\sum_{t=1}^T (\ell_t(\pi_t) - C_t(\pi_t)) + \min_{\pi \in \tilde{\Pi}} \sum_{t=1}^T C_t(\pi) - \min_{\pi' \in \tilde{\Pi}} \sum_{t=1}^T \ell_t(\pi')]$ is the  ``convex optimization gap'' that measures how close the surrogate losses $C_t$ is to minimizing the cost-sensitive losses $\ell_t$.
\end{corSCPDef}
\begin{proof}
Follows immediately from Theorem \ref{thmSCP} using the definition of $R$, $\tilde{R}$ and $\mathcal{G}$, since $\mathcal{G} = \frac{R - \tilde{R}}{T}$
\end{proof}

\clearpage

\bibliography{references}
\bibliographystyle{icml2013}

\end{document}

%% file: abstractBagnell.tex
Many prediction domains, such as ad placement, recommendation, trajectory prediction, and document summarization, require predicting a \emph{set} or \emph{list} of options.
Such lists are often evaluated using submodular reward functions that measure both quality and diversity.
We propose a simple, efficient, and provably near-optimal approach to optimizing such prediction problems based on no-regret learning.
Our method leverages a surprising result from online submodular optimization: a single no-regret online learner can compete with an optimal \emph{sequence} of predictions. 
Compared to previous work, which either learn a sequence of classifiers or rely on stronger assumptions such as realizability, 
we ensure both data-efficiency as well as performance guarantees in the fully agnostic setting. Experiments validate the efficiency and applicability of the approach
on a wide range of problems including manipulator trajectory optimization, news recommendation and document summarization.

%% file: introduction.tex
Many problem domains, ranging from web applications such as ad placement or
content recommendation to identifying successful robotic grasp trajectories
require predicting lists of items. Such applications are often budget-limited
and the goal is to choose the best list of items, from a large set of possible items, with maximal utility.  In ad placement, we must pick a
small set of ads with high click-through rate. For robotic manipulation, we must pick a small set of initial grasp trajectories to maximize the chance of finding a successful trajectory via more extensive evaluation or simulation.\todo{JAB: I actually kind of like the old one better. LOTS OF TYPOS got introduced. I fixed the ones I saw, but
please go over carefully.}

In all of these problems, the predicted list of items should be both relevant
and diverse. For example, recommending a diverse set of news articles
increases the chance that a user would like at least one article
\cite{Radlinski2008ranked}.  As such, 
recommending multiple redundant articles on the same topic would do little to
increase this chance.  This notion
of diminishing returns due to redundancy is often captured formally using submodularity
\cite{guestrin08submodtut}. 

Exact submodular function optimization is intractable, but simple greedy selection is known to have strong near-optimal performance guarantees and typically works very well in practice \cite{guestrin08submodtut}.  Given access to the submodular reward function, one could simply employ greedy to construct good lists. 

%

In this paper, we study the general supervised learning problem of training a policy to maximize a submodular reward function.  We assume that the submodular reward function is only directly measured on a finite training set, and our goal is to learn to make good predictions on new test examples where the reward function is not directly measurable.

We develop a novel agnostic learning
approach based on new analysis showing that a \textbf{single} no-regret learner 
can produce a near-optimal \emph{list} of predictions.\footnote{This result
may seem surprising given that previous approaches \cite{streeter2007online}
require a sequence of online learners -- one for each position in the list.} We use a reduction approach to ``lift'' this result to contextual hypothesis
classes that map features to predictions, and bound performance relative to the
optimal sequence of hypotheses in the class. In contrast to previous work, our
approach ensures both data-efficiency as well as performance guarantees in the
fully agnostic setting. Moreover, our approach is simple to implement and easily
integrates with conventional off-the-shelf learning algorithms. Empirical evaluations show  
our approach to be competitive with or exceed the
state-of-the-art performance on a variety of problems, ranging from trajectory
prediction in robotics to extractive document summarization. 




%% file: related.tex
\section{Related Work}


The problem of learning to optimize submodular
reward functions from data, both with and without contextual features, has become
increasingly important in machine learning due to its diverse application areas. Broadly speaking, there
are two main approaches for this setting. The first aims to identify a model
within a parametric family of submodular functions and then use the resulting model for new predictions. The second attempts to learn a strategy to directly
predict a list of elements by decomposing the overall problem into multiple simpler
learning tasks.



The first approach \cite{Yue08svmdiv,Yue11LinearSubmod,lin2012learning,raman2012online}
involves identifying the parameterization that best matches the submodular
rewards of the training instances. These methods are largely limited
to learning non-negative linear combinations of features that are themselves
submodular, which often restricts their expressiveness. 
Furthermore, while good sample complexity results are known,
these guarantees only hold under strong realizability assumptions
where submodular rewards can be modeled exactly by such linear
combinations \cite{Yue11LinearSubmod,raman2012online}. 
Recent work on \emph{Determinental Point Processes} (DPPs) \cite{kulesza2012learning} provide a 
probabilistic model of sets, which can be useful for the tasks that we consider.
These approaches, while appealing,
solve a potentially unnecessarily hard problem in first learning a holistic list evaluation model, and thus may compound errors by first approximating the
submodular function and then approximately optimizing it.


The second, a learning reduction approach, by contrast, decomposes list prediction into a sequence of simpler learning tasks that attempts to mimic the greedy strategy \cite{streeter2007online,Radlinski2008ranked,streeter2009online,deyConSeqOptRSS}. 
In \cite{deyConSeqOptRSS}, this strategy was extended to the contextual setting
by a reduction to cost-sensitive classification.
Essentially,
each learning problem aims to best predict an item to add to the list, given features,
so as to maximize the expected marginal utility. 
This approach is flexible, in that it can be used with most
common hypothesis classes and arbitrary features.
Because of this decomposition, the
full model class (all possible sequences of predictors) is often quite
expressive, and allows for agnostic learning guarantees.\footnote{This first strategy of learning the parameters of a submodular function can be seen as a special case of this second approach (see section \ref{secCSCReduction}).} This generality comes at the 
expense of being significantly less data-efficient than methods that make realizability
assumptions such as \cite{Yue11LinearSubmod,raman2012online}, as the existing
approach learns a \emph{different} classifier for each position in the list.


Compared with related work, our approach enjoys the benefits of being both
data-efficient while ensuring strong agnostic performance guarantees. We do so by developing new analysis for online submodular
optimization which yields agnostic learning guarantees while learning a 
\textbf{single} data-efficient policy.

%% file: background.tex
Let $\mathcal{S}$ denote the set of possible items to choose from
(\textbf{e.g.} ads, sentences, grasps). Our objective is to pick a list of
items $L \subseteq \mathcal{S}$ to maximize a reward function $f$ that obeys the following properties:\footnote{``Lists'' generalize the notion of ``set'' more commonly used in submodular optimization, and enables reasoning about item order and repeated items \cite{streeter2007online}. One may consider sets where appropriate.
}
\vspace{-0.1in}
\begin{enumerate}
\item \textbf{Monotonicity:} For any lists $L_1, L_2$, $f(L_1) \leq f(L_1 \oplus L_2)$ and $f(L_2) \leq f(L_1 \oplus L_2)$
\item \textbf{Submodularity:} For any lists $L_1,L_2$ and item $s \in \mathcal{S}$, $f(L_1 \oplus s) - f(L_1) \geq f(L_1 \oplus L_2 \oplus s) - f(L_1 \oplus L_2)$.
\end{enumerate}
\vspace{-0.1in}
Here, $\oplus$ denotes the concatenation operator. Intuitively, monotonicity
implies that adding more elements never hurts, and
submodularity captures the notion of diminishing returns (\emph{i.e.} adding an item
to a long list increases the objective less than when adding it to a shorter sublist). 
We further assume for simplicity that $f$ takes values in 
$[0,1]$, and that $f(\emptyset) = 0$ where $\emptyset$ denotes the empty
list. We will also use the shorthand $b(s|L) = f(L \oplus s) - f(L)$ to denote
the marginal benefit of adding the item $s$ to list $L$.

A simple example submodular function that repeatedly arises in many domains is
one that takes value $0$ until a suitable instance is found,
and then takes on value $1$ thereafter. Examples include 
the notion of ``multiple choice'' learning as in
\cite{deyConSeqOptRSS,guzman2012multiple} where a predicted set of options 
is considered successful if any predicted item is deemed correct,
and abandonment in ad placement \cite{Radlinski2008ranked} where success is measured by whether any predicted advertisement
is clicked on.

We consider reward functions that may depend on
some underlying state $x \in \mathcal{X}$ (e.g. a user,
environment of the robot, a document, etc.). 
Let $f_x$ denote the
reward function for state $x$, and assume that 
$f_x$ is monotone submodular for all $x$.   

\subsection{Learning Problem}
Our task consists in learning to construct good lists of pre-specified length $k$ under some unknown
distribution of states $D$ (e.g. distribution of users or documents we have to
summarize). We consider two cases: context-free and contextual.

\textbf{Context-Free.}
In the context-free case, we have no side-information about the current state (i.e. we do not observe anything about $x$).  We quantify the performance of any list $L$ by its expected value:
$$F(L) = \mathbb{E}_{x  \sim D}[f_x(L)].$$
Note that $F(L)$ is also monotone submodular. Thus the clairvoyant greedy
algorithm with perfect knowledge of $D$ can find a list $\hat{L}_k$ such that $F(\hat{L}_k) \geq
(1-1/e) F(L^*_k)$, were $L^*_k = \argmax_{L:|L|=k} F(L)$. 
Although $D$ is unknown, we assume that we observe
samples of the objective $f_x$ during training.  Our goal is thus to develop a learning approach that efficiently converges, both computationally and statistically, to the performance of the clairvoyant greedy algorithm.


\textbf{Contextual.}
In the contextual case, we observe side-information in the form of features regarding the state of the world.  We ``lift''  this problem to a hypothesis space of policies  (\emph{i.e.} multi-class predictors) that map
features to items. 

Let $\Pi$ denote our policy class, and let $\pi(x)$ denote the prediction of policy $\pi \in \Pi$ given side-information describing state $x$.
Let $L_{\pi,k} = (\pi_1,\pi_2,\dots,\pi_k)$ denote a list of policies. In state $x$, this list of policies will predict 
$L_{\pi,k}(x) = (\pi_1(x),\pi_2(x),\dots,\pi_k(x))$. We quantify performance using the expected value:
$$F(L_{\pi}) = \mathbb{E}_{x \sim  D}[f_x(L_\pi(x))].$$ 
It can be shown that $F$ obeys both monotonicity and submodularity with respect to appending policies \cite{deyConSeqOptRSS}. Thus, a clairvoyant
greedy algorithm that sequentially picks the \emph{policy} with highest expected benefit
will construct a list $\hat{L}_{\pi,k}$ such that
$F(\hat{L}_{\pi,k}) \geq (1-1/e) F(L^*_{\pi,k})$, where $L^*_{\pi,k} = \argmax_{L_{\pi}:|L_{\pi}| = k} F(L_{\pi})$.
As before, our goal is to develop a learning approach (for learning a list of policies) that \emph{efficiently} competes with the performance of the clairvoyant greedy algorithm.

%% file: ContextFree.tex
\begin{algorithm}[t]
\begin{small}
\begin{algorithmic}
\STATE \textbf{Input:} Set of items $\mathcal{S}$, length $m$ of list to construct, length $k$ of best list to compete against, online learner \textsc{Predict} and \textsc{Update} functions.
\FOR{$t=1$ \textbf{to} $T$}
\STATE Call online learner \textsc{Predict}() $m$ times to construct list $L_t$. (e.g. by sampling $m$ times from online learner's internal distribution over items).
\STATE Evaluate list $L_t$ on a sampled state $x_t \sim D$.
\STATE For all $s \in \mathcal{S}$, define its discounted cumulative benefit: $r_t(s) = \sum_{i=1}^m (1-1/k)^{m-i} b(s|L_{t,i-1},x_t)$.
\STATE For all $s \in \mathcal{S}$: define loss $\ell_t(s) = \max_{s' \in \mathcal{S}} r_t(s') - r_t(s)$
\STATE Call online learner update with loss $\ell_t$: $\textsc{Update}(\ell_t)$
\ENDFOR
\end{algorithmic}
\end{small}
\caption{Submodular Contextual Policy (SCP) Algorithm in context-free setting.\label{algSCPNoContext}}
\end{algorithm}

\section{Context-free List Optimization}


%


We first consider the context-free setting. Our algorithm, called Submodular Contextual Policy (SCP), is described in Algorithm \ref{algSCPNoContext}.  
SCP
requires an online learning algorithm subroutine (denoted by \textsc{Update}) 
   that is no-regret with respect to a bounded positive loss
function,\footnote{See Section \ref{sec:context_free_theory} and \eqref{eqn:regret} for a definition of no-regret.} maintains an internal distribution over items for prediction, and can be queried for multiple predictions (i.e. multiple samples).\footnote{Algorithms that meet these requirements include Randomized Weighted Majority \cite{littlestone1994weighted}, Follow the Leader \cite{kalai2005efficient}, 
 EXP3 \cite{auer2003nonstochastic}, and many others.} 
In contrast to prior work \cite{streeter2007online}, SCP employs only a \textit{single} online
learning in the inner loop. 
 
SCP proceeds by training over a sequence of states $x_1, x_2, \dots, x_T$.
At each iteration, SCP queries the online learner to generate a list of $m$ items (via \textsc{Predict}, e.g. by sampling from its internal distribution over items), evaluates a weighted cumulative benefit of each item on the
sampled list to define a loss related to each item, and then uses the online learner (via \textsc{Update}) to update its internal distribution.  

During training, we allow the algorithm to construct lists of length $m$, rather than $k$. 
In its simplest form, one may simply choose $m=k$. However, it may be beneficial to choose $m$ differently than $k$, as is shown later in the theoretical
analysis. 

Perhaps the most unusual aspect is how loss is defined using the weighted cumulative benefits of each item: 
\begin{eqnarray}r_t(s) = \sum_{i=1}^m (1-1/k)^{m-i} b(s|L_{t,i-1},x_t),\label{eqn:weighted_benefits}\end{eqnarray}
where $L_{t,i-1}$ denotes the first $i-1$ items in $L_t$, and 
\begin{eqnarray}b(s|L_{t,i-1},x_t) = f_{x_t}(L_{t,i-1} \oplus s) - f_{x_t}(L_{t,i-1}).\label{eqn:b_context_free}\end{eqnarray}
Intuitively,
\eqref{eqn:weighted_benefits} represents the  weighted sum of benefits of item $s$ in state $x_t$ had we added it at any intermediate stage in $L_t$. 
The benefits at different positions are weighed differently, where position $i$ is adjusted by a factor $(1-1/k)^{m-i}$. These weights are derived via our theoretical analysis,
 and indicate that benefits in early positions should be more discounted than benefits in later positions. 
 Intuitively, this weighting has the effect of rebalancing the benefits so that each position contributes more equally to the overall loss.\footnote{We also consider a similar algorithm in the min-sum cover setting, where the theory also requires reweighting benefits, but instead weights earlier benefits more
highly (by a factor $m-i$, rather than $(1-1/k)^{m-i}$). We omit discussing this variant for brevity.}

SCP requires the ability to directly measure $f_{x}$ in each training instance $x_t$.  Directly measuring $f_{x_t}$ enables us to obtain loss measurements $\ell_t(s)$ for any $s \in \mathcal{S}$. 
For example, in document summarization $f_x$ corresponds to the  ROUGE score \cite{lin2004rouge}, which can be evaluated for any generated summary given expert annotations which are only available for training instances. 

In principle, SCP can also be applied in partial feedback settings, e.g. ad placement where the value $f_{x_t}$ is only observed for some items (\emph{e.g.} only the displayed ads), by using bandit learning algorithms instead (e.g. EXP3 \cite{auer2003nonstochastic}).\footnote{Partial information settings arise, e.g., when $f$ is derived using real-world trials that preclude the ability to evaluate $b(s|L,x)$ \eqref{eqn:b_context_free} for every possible  $s\in\mathcal{S}$.}  As this is an orthogonal issue, most of our focus is on the full information case.\todo{JAB: Is this footnote useful?}

\subsection{Theoretical Guarantees}
\label{sec:context_free_theory}

We now show that Algorithm \ref{algSCPNoContext} is no-regret with respect to the clairvoyant greedy algorithm's expected performance over the training instances. 
Our main theoretical result provides a reduction to an online learning problem and directly relates the performance of our algorithm on the submodular list optimization problem to the standard online learning regret incurred by the subroutine. 

Although Algorithm \ref{algSCPNoContext} uses only a \textit{single} instance of an online learner subroutine, it achieves the same performance guarantee as prior work \cite{streeter2007online,deyConSeqOptRSS} that employ $k$ separate instances of an online learner.
This leads to a surprising fact: it is possible to sample from a stationary distribution over items to construct a list that achieves the same guarantee as the clairvoyant greedy algorithm. \footnote{This fact can also be seen as a special case of a more general result proven in prior related work that analyzed randomized set selection strategies to optimize submodular functions \cite{feige2011}.} 

For a sequence of training states $\{ x_t \}_{t=1}^T$, let the sequence of loss functions $\{ \ell_t  \}_{t=1}^T$ defined in Algorithm \ref{algSCPNoContext} correspond to the sequence of losses incurred in the reduction to the online learning problem. The expected regret of the online learning algorithm is 
\begin{eqnarray}\mathbb{E}[R] = \sum_{t=1}^T \mathbb{E}_{s' \sim p_t}[\ell_t(s')] - \min_{s \in \mathcal{S}} \sum_{t=1}^T \ell_t(s),\label{eqn:regret}\end{eqnarray}
where $p_t$ is the internal distribution of the online learner used to construct list $L_t$. Note that an online learner is called \textit{no-regret} if $R$ is sublinear in $T$.

Let $F(p,m) = \mathbb{E}_{L_m \sim p}[ \mathbb{E}_{x \sim D} [ f_x(L_m) ] ]$ denote the expected value of constructing lists by sampling (with replacement) $m$ elements from distribution $p$, and let $\hat{p} = \arg\max_{t \in \{1,2,\dots,T\}} F(p_t,m)$ denote the best distribution found by the algorithm.  

We define a mixture distribution $\overline{p}$ over lists that constructs a list as follows: sample an index $t$ uniformly in $\{1,2,\dots,T\}$, then sample $m$ elements (with replacement) from $p_t$. 
Note that $F(\overline{p},m) = \frac{1}{T} \sum_{t=1}^T F(p_t,m)$ and $F(\hat{p},m) \geq F(\overline{p},m)$.  Thus it suffices to show that $F(\overline{p},m)$ has good guarantees. 
We show that in expectation $\overline{p}$ (and thus $\hat{p}$) constructs lists with performance guarantees close to the clairvoyant greedy algorithm:\footnote{Additionally, if the distributions $p_t$ converge, then the last distribution $p_{T+1}$ must have performance arbitrarily close to $\overline{p}$ as $T \rightarrow \infty$. In particular, we can expect this to occur when the examples are randomly drawn from a fixed distribution that does not change over time.}

\begin{theorem} \label{thmSCPNoContext}
Let $\alpha = \exp(-m/k)$ and $k' = \min(m,k)$. For any $\delta \in (0,1)$, with probability $\geq 1-\delta$: 
\begin{displaymath}
F(\overline{p},m) \geq (1-\alpha) F(L^*_k) - \frac{\mathbb{E}[R]}{T} - 3\sqrt{\frac{2k'\ln(2/\delta)}{T}}
\end{displaymath}
\end{theorem}


\begin{corollary} 
If a no-regret algorithm is used on the sequence of loss $\ell_t$, then as $T \rightarrow \infty$, $\frac{\mathbb{E}[R]}{T} \rightarrow 0$, and:
\begin{displaymath}
\lim_{T \rightarrow \infty} F(\overline{p},m) \geq (1-\alpha) F(L^*_k) 
\end{displaymath}
\end{corollary}

Theorem \ref{thmSCPNoContext} provides a general approximation ratio to the best list of size $k$ when constructing a list of a different size $m$. For $m=k$, we obtain the typical $(1-1/e)$ approximation ratio \cite{guestrin08submodtut}.  As $m$ increases, this provides approximation ratios that converge exponentially closer to 1. 


Naively, one might expect regret $\mathbb{E}[R]/T$ to scale linearly in $k'$ as it involves loss in $[0,k']$. However, we show that regret  actually scales as $O(\sqrt{k'}$) (e.g. using Weighted Majority \cite{kalai2005efficient}).
Our result matches the best known results for this setting \cite{streeter2007online} while using a \textit{single} online learner, and is especially beneficial in the contextual setting due to improved generalization (see Section \ref{sec:contextual}).

\vspace{0.05in}
\begin{corollary} \label{corWM}
Using weighted majority with the optimal learning rate  guarantees with probability $\geq 1-\delta$:
\begin{small}
\begin{displaymath}
F(\overline{p},m) \geq (1-\alpha) F(L^*_k) - O\left(\sqrt{\frac{k'\log(1/\delta)}{T}} + \sqrt{\frac{k' \log|\mathcal{S}|}{T}}\right).
\end{displaymath}
\end{small}
\end{corollary}

%% file: Contextual.tex
We now consider the contextual setting where features of each state $x_t$ are observed before choosing the list. As mentioned, our goal here is to compete with the best list of policies $(\pi_1,\pi_2,\dots,\pi_k)$ from a hypothesis class $\Pi$. Each of these policies are assumed to choose an item solely based on features of the state $x_t$. 

We consider embedding $\Pi$ within a larger class, $\Pi \subseteq \tilde{\Pi}$, where policies $\tilde{\Pi}$ are functions of both state and a partially chosen list. Then for any $\pi \in \tilde{\Pi}$, $\pi(x,L)$ corresponds to the item that policy $\pi$ selects to append to list $L$ given state $x$. We will learn a policy, or distribution of policies, from $\tilde{\Pi}$ that attempts to generalize list construction across multiple positions.\footnote{Competing against the best list of policies in $\tilde{\Pi}$ is difficult in general as it violates submodularity: policies can perform better when added later in the list (due to list features). Nevertheless, we can still learn from class $\tilde{\Pi}$ and compete against the best list of policies in $\Pi$.}

\begin{algorithm}[t]
\begin{small}
\begin{algorithmic}
\STATE \textbf{Input:} Set of items $\mathcal{S}$, policy class $\tilde{\Pi}$, length $m$ of list we construct, length $k$ of best list we compete against.
\STATE Pick initial policy $\pi_1$ (or distribution over policies)
\FOR{$t=1$ \textbf{to} $T$}
\STATE Observe features of a sampled state $x_t \sim D$ (e.g. features of user/document)
\STATE Construct list $L_t$ of $m$ items using $\pi_{t}$ with features of $x_t$ (or by sampling a policy for each position if $\pi_t$ is a distribution over policies).
\STATE Define $m$ new cost-sensitive classification examples $\{(v_{ti},c_{ti},w_{ti})\}_{i=1}^m$ where:  
\begin{enumerate}
\item $v_{ti}$ is the feature vector of state $x_{t}$ and list $L_{t,i-1}$
\item $c_{ti}$ is the cost vector such that $\forall s \in \mathcal{S}$: $c_{ti}(s) =  \max_{s' \in \mathcal{S}} b(s'|L_{t,i-1},x_t) - b(s|L_{t,i-1},x_t)$
\item $w_{ti} = (1-1/k)^{m-i}$ is the weight of this example
\end{enumerate}
\STATE $\pi_{t+1} = \textsc{Update}(\pi_t, \{(v_{ti},c_{ti},w_{ti})\}_{i=1}^m)$
\ENDFOR
\STATE \textbf{return} $\pi_{T+1}$
\end{algorithmic}
\end{small}
\caption{Submodular Contextual Policy (SCP) Algorithm. \label{algSCP}}
\end{algorithm}



We present an extension of SCP to the contextual setting (Algorithm \ref{algSCP}). At each iteration, SCP constructs a list $L_t$ for the state $x_t$ (using its current policy or by sampling policies from its distribution over policies). 
   
Analogous to the context-free setting, we define a loss function for the learner subroutine (\textsc{Update}).  
We represent the loss using weighted cost-sensitive classification examples $\{(v_{ti},c_{ti},w_{ti})\}_{i=1}^m$, where $v_{ti}$ denotes features of the state $x_t$ and list $L_{t,i-1}$, $w_{ti}=(1-1/k)^{m-i}$ is the weight associated to this example, and $c_{ti}$ is the cost vector specifying the cost of each item $s \in \mathcal{S}$
\begin{eqnarray}c_{ti}(s) = \max_{s' \in \mathcal{S}} b(s'|L_{t,i-1},x_t)- b(s|L_{t,i-1},x_t).\label{eqn:cost_contextual}\end{eqnarray}
The loss incurred by any policy $\pi$ is defined by its loss on this set of cost-sensitive classification examples, i.e. 
$$\ell_t(\pi) = \sum_{i=1}^m w_{ti} c_{ti}(\pi(v_{ti})).$$ 
These new examples are then used to update the policy (or distribution over policies) using a no-regret algorithm (\textsc{Update}). This reduction effectively transforms the task of learning a policy for this submodular list optimization problem into a standard online cost-sensitive classification problem.\footnote{This is similar to DAgger \cite{Ross11,Ross11b,Ross12} developed for sequential prediction problems like imitation learning. Our work can be seen as a specialization of DAgger for submodular list optimization, and ensures that we learn policies that pick good items under the lists they construct. Unlike prior work, our analysis leverages submodularity, leading to several modifications, and improved global optimality guarantees.}  Analogous to the context-free setting, we can also extend to partial feedback settings where $f$ is only partially measurable by using contextual bandit algorithms such as EXP4 \cite{auer2003nonstochastic} as the online learner (\textsc{Update}).\footnote{Analogous to the context-free setting, partial information arises when $c_{ti}$ \eqref{eqn:cost_contextual} is not measurable for every $s\in\mathcal{S}$.}\todo{SR: Added last sentence here to say we can handle partial feedback in contextual case with SCP, and tweaked the footnote so it still fits in 8 pages}


\subsection{No-Regret Cost-Sensitive Classification} \label{secCSCReduction}

Having transformed our problem into online cost-sensitive classification, we now present approaches that can be used to achieve no-regret on such tasks. 
For finite policy classes $\tilde{\Pi}$, one can again leverage any no-regret online algorithm such as Weighted Majority \cite{kalai2005efficient}. Weighted Majority maintains a distribution over policies in $\tilde{\Pi}$ based on the loss $\ell_t(\pi)$ of each $\pi$, and achieves regret at a rate of 
$$R=\sqrt{k'\log|\tilde{\Pi}|/T},$$ 
for $k'=\min(m,k)$. In fact, the context-free setting can be seen as a special case, where 
$\Pi = \tilde{\Pi} = \{ \pi_s | s \in \mathcal{S} \}$ and $\pi_s(v) = s$ for any $v$.

However, achieving no-regret for infinite policy classes is in general not tractable. A more practical approach is to employ existing reductions of cost-sensitive classification problems to convex optimization problems, for which we can efficiently run no-regret convex optimization (e.g. gradient descent). These reductions effectively upper bound the cost-sensitive loss by a convex loss, and thus bound the original loss of the list prediction problem. We briefly describe two such reductions from \cite{beygelzimer2005error}:
\todo{SR: Cited beygelzimmer in last sentence so we don't claim these as new reductions, tweaked slightly the below so it would still fit in 8 pages.}


\paragraph{Reduction to Regression}
We transform cost-sensitive classification into a regression problem of predicting the costs of each item $s \in \mathcal{S}$. Afterwards, the policy chooses the item with lowest predicted cost. We convert each weighted cost-sensitive example $(v,c,w)$ into $|\mathcal{S}|$ weighted regression examples. 

For example, if we use least-squares linear regression, the weighted squared loss for a particular example $(v,c,w)$ and policy $h$ would be: 
$$\ell(h) = w \sum_{s \in \mathcal{S}} (h^\top v(s) - c(s))^2.$$ 

\paragraph{Reduction to Ranking}
Another useful reduction transforms the problem into a ''ranking'' problem that penalizes ranking an item $s$ above another better item $s'$.  In our experiments, we employ a weighted hinge loss, and so the penalty is proportional to the difference in cost of the misranked pair. 
   For each cost-sensitive example $(v,c,w)$, we generate $\mathcal{|S|}(\mathcal{|S|} -1)/2$ ranking examples for every distinct pair of items $(s,s')$, where we must predict the best item among $(s, s')$ (potentially by a margin), with a weight of $w|c(s)-c(s')|$. 

   For example, if we train a linear SVM \cite{joachims2005support}, we obtain a weighted hinge loss of the form: 
 $$w |\delta_{s,s'}| \max( 0, 1 - h^\top (v(s)-v(s')) \textrm{sign}(\delta_{s,s'}) ),$$ 
 where $\delta_{s,s'} = c(s) - c(s')$ and $h$ is the linear policy.
At prediction time, we simply predict the item $s^*$ with highest score, $s^* = \argmax_{s \in \mathcal{S}} h^\top v(s)$. 
This reduction proves advantageous whenever it is easier to predict pairwise rankings rather than the actual cost. 

\subsection{Theoretical Guarantees}


We now present contextual performance guarantees for SCP that relate performance on the submodular list optimization task to the regret of the corresponding online cost-sensitive classification task.
Let  $\ell_t : \tilde{\Pi} \rightarrow \mathbb{R}$ compute the loss of each policy $\pi$ on the cost-sensitive classification examples $\{v_{ti},c_{ti},w_{ti}\}_{i=1}^m$ collected in Algorithm \ref{algSCP} for state $x_t$.
We use  $\{ \ell_t  \}_{t=1}^T$ as the sequence of losses for the online learning problem. 

For a deterministic online algorithm that picks the sequence of policies $\{\pi_t\}_{t=1}^T$, the regret is
$$R = \sum_{t=1}^T \ell_t(\pi_t) - \min_{\pi \in \tilde{\Pi}} \sum_{t=1}^T \ell_t(\pi).$$
For a randomized online learner, let $\pi_t$ be the distribution over policies at iteration $t$, with expected regret
$$\mathbb{E}[R] = \sum_{t=1}^T \mathbb{E}_{\pi'_t \sim \pi_t}[\ell_t(\pi'_t)] - \min_{\pi \in \tilde{\Pi}} \sum_{t=1}^T \ell_t(\pi).$$
Let $F(\pi,m) = \mathbb{E}_{L_{\pi,m} \sim \pi}[ \mathbb{E}_{x \sim D} [ f_x(L_{\pi,m}(x)) ] ]$ denote the expected value of constructing lists by sampling (with replacement) $m$ policies from distribution $\pi$ (if $\pi$ is a deterministic policy, then this means we use the same policy at each position in the list). Let $\hat{\pi} = \argmax_{t \in \{1,2,\dots,T\}} F(\pi_t,m)$ denote the best distribution found by the algorithm in hindsight. 

We use a mixture distribution $\overline{\pi}$ over policies to construct a list as follows: sample an index $t$ uniformly in $\{1,2,\dots,T\}$, then sample $m$ policies from $\pi_t$ to construct the list.  As before, we note that $F(\overline{\pi},m) = \frac{1}{T} \sum_{t=1}^T F(\pi_t,m)$, and $F(\hat{\pi},m) \geq F(\overline{\pi},m)$. As such, we again focus on proving good guarantees for $F(\overline{\pi},m)$, as shown by the following theorem.
\begin{theorem}\label{thmSCP}
Let $\alpha = \exp(-m/k)$, $k' = \min(m,k)$ and pick any $\delta \in (0,1)$. After $T$ iterations, for deterministic online algorithms, we have that with probability at least $1-\delta$:
\begin{displaymath}
F(\overline{\pi},m) \geq (1-\alpha) F(L^*_{\pi,k}) - \frac{R}{T} - 2 \sqrt{\frac{2 \ln(1/\delta)}{T}}.
\end{displaymath}
Similarly, for randomized online algorithms, with probability at least $1-\delta$:
\begin{displaymath}
F(\overline{\pi},m) \geq (1-\alpha) F(L^*_{\pi,k}) - \frac{\mathbb{E}[R]}{T}  - 3\sqrt{\frac{2k'\ln(2/\delta)}{T}}.
\end{displaymath}
\end{theorem}
Thus, as in the previous section, a no-regret algorithm must achieve $F(\overline{\pi},m) \geq (1-\alpha) F(L^*_{\pi,k})$ with high probability as $T \rightarrow \infty$. This matches similar guarantees provided in \cite{deyConSeqOptRSS}. 
   Despite having similar guarantees, we intuitively expect SCP to outperform \cite{deyConSeqOptRSS} in practice because SCP can use all data to train a \textit{single} predictor, instead of being split to train $k$ separate ones. We empirically verify this intuition in Section \ref{sec:experiments}.

%


When using surrogate convex loss functions (such as regression or ranking loss),
we provide a general result that applies if the online learner uses any convex upper bound of the cost-sensitive loss. An extra penalty term is introduced that relates the gap between the convex upper bound and the original cost-sensitive loss:
\begin{corollary}\label{corSCP}
Let $\alpha = \exp(-m/k)$ and $k' = \min(m,k)$. If we run an online learning algorithm on the sequence of convex loss $C_t$ instead of $\ell_t$, then after $T$ iterations, for any $\delta \in (0,1)$, we have that with probability at least $1-\delta$:
\begin{displaymath}
F(\overline{\pi},m) \geq (1-\alpha) F(L^*_{\pi,k}) - \frac{\tilde{R}}{T} - 2 \sqrt{\frac{2 \ln(1/\delta)}{T}} - \mathcal{G}
\end{displaymath}
where $\tilde{R}$ is the regret on the sequence of convex loss $C_t$, and  $\mathcal{G}$ is defined as 
   \begin{small}
 \begin{align*}
&\frac{1}{T}\left[\sum_{t=1}^T (\ell_t(\pi_t) - C_t(\pi_t)) + \min_{\pi \in \tilde{\Pi}} \sum_{t=1}^T C_t(\pi) - \min_{\pi' \in \tilde{\Pi}} \sum_{t=1}^T \ell_t(\pi')\right]
\end{align*}
\end{small}
\hspace{-0.06in}and denotes the  ``convex optimization gap'' that measures how close the surrogate $C_t$ is to minimizing $\ell_t$.

\end{corollary}

This result implies that using a good surrogate convex loss for no-regret convex optimization will lead to a policy that has a good performance relative to the optimal list of policies.
Note that the gap $\mathcal{G}$ often may be small or non-existent. For instance, in the case of the reduction to regression or ranking, $\mathcal{G} = 0$ in realizable settings where there exists a ``perfect'' predictor in the class.  
 Similarly, in cases where the problem is near-realizable we would expect $\mathcal{G}$ to be small.\footnote{We conjecture that this gap term $\mathcal{G}$ is not specific to our particular scenario, but rather is (implicitly) always present whenever one attempts to optimize classification accuracy via surrogate convex optimization.}

%% file: traj_opt_sec.tex
\todo{YY: I changed the introduction of this experimental setup to focus more on the learning task, rather than the overall task of grasp trajectory.  This is more in line with the other experiment descriptions.}
We applied SCP to a manipulation planning task for a $7$
degree-of-freedom robot manipulator.  The goal is to predict a set of initial trajectories so as to maximize the chance that one of them leads to a collision-free trajectory.  We use local trajectory optimization techniques such as CHOMP \cite{ratliff2009chomp}, which have proven effective in quickly finding collision-free trajectories using local perturbations of an initial trajectory.  Note that selecting a diverse set of initial trajectories is important since local techniques such as CHOMP often get stuck in local optima.\footnote{I.e., similar or redundant inital trajectories will lead to the same local optima.}



We use the dataset from \cite{deyConSeqOptRSS}. It consists of
$310$ training and $212$ test environments of random obstacle
configurations around a target object, and $30$ initial seed trajectories. In each environment, each seed trajectory has $17$ features describing the spatial properties of the trajectory relative to obstacles.\footnote{In addition to the base features, we add features of the
current list w.r.t. each initial trajectory. We use the per
feature minimum absolute distance and average absolute value of the distance
to the features of initial trajectories  in the list. We also use a bias feature
always set to $1$, and an indicator feature which is $1$ when selecting the element in
the first position, $0$ otherwise.}



Following \cite{deyConSeqOptRSS}, we employ a reduction of
cost-sensitive classification to regression as explained in
Section \ref{secCSCReduction}. We compare SCP to ConSeqOpt \cite{deyConSeqOptRSS} (which learns $k$ separate predictors), and Regression (regress success rate from features to sort seeds; this accounts for relevance but not diversity). 


Figure \ref{all_figs} (left) shows the failure probability over the test
environments versus the number of training environments. 
ConSeqOpt employs a reduction to $k$ classifiers.  As a consequence, ConSeqOpt faces data starvation issues for small training sizes, as there is little data available for training predictors lower in the list.\footnote{When a successful seed is found, benefits at later positions are 0. This effectively discards training environments for training classifiers lower in the list in ConSeqOpt.} \todo{SR: added footnote here as I don't think we explained what data starvation was in this new version. Also slightly tweaked the above 2 paragraphs so that it fits in 8 pages}
In contrast, SCP has no data starvation issue and outperforms both ConSeqOpt and
Regression.

%% file: news.tex

We built a stochastic user simulation based on $75$ user
preferences derived from a user study in 
\cite{Yue11LinearSubmod}. 
Using this simulation as a training oracle, our goal is to learn 
to recommend articles to any user (depending on their contextual features) to minimize the failure case where the user does not like any of the recommendations.\footnote{Also known as abandonment \cite{Radlinski2008ranked}.}

Articles are represented by features, and user preferences by linear weights.
We derived user contexts by soft-clustering users into groups, and using corrupted group memberships as contexts.

We perform five-fold cross validation. In each fold, we train SCP and ConSeqOpt on $40$ users'
preferences, use $20$ users for validation, and then test on the held-out $15$ users. Training, validation and testing are all performed via simulation. 
Figure \ref{all_figs} (middle) shows the results, where we see the recommendations made by 
SCP achieves significantly lower failure rate as the number of recommendations is increased from $1$ to $5$.

%% file: docsum.tex
In the extractive multi-document summarization task, the goal is to extract sentences (with character budget $B$) to maximize coverage of human-annotated summaries.
Following the experimental setup from \cite{lin2010multi} and \cite{kulesza2012learning}, we use data from the Document
Understanding Conference (DUC) 2003 and 2004 (Task 2) \cite{dang2005overview}. Each training or test instance corresponds to a cluster of documents, 
and contains
approximately $10$ documents belonging to the same topic and four human reference summaries. We train on the 2003 data (30 clusters) and test on the 2004 data (50 clusters). The budget is $B= 665$ bytes, including spaces.

We use the ROUGE \cite{lin2004rouge} unigram statistics (ROUGE-1R, ROUGE-1P,
ROUGE-1F) for performance evaluation. Our method directly attempts to optimize the
ROUGE-1R objective with respect to the reference summaries, which can be easily
shown to be monotone submodular \cite{lin2011class}.



We aim to predict sentences that are both short and informative. 
Therefore we maximize the normalized marginal benefit,
\begin{eqnarray}b'(s|L_{t,i-1}) = b(s|L_{t,i-1})/l(s),\label{eqn:normalized_benefit}\end{eqnarray}
where $l(s)$ is the length of the sentence $s$.\footnote{This results in a knapsack constrained optimization problem.  We expect our approach to perform well in this setting, but defer a formal analysis for future work.} We use a reduction to ranking as described in Section \ref{secCSCReduction} using \eqref{eqn:normalized_benefit}.
While not performance-optimized, our approach takes less than $15$ minutes to train.

Following \cite{kulesza2012learning}, we consider features $f_i$ for each sentence
consisting of \emph{quality features} $q_i$ and \emph{similarity features} $\phi_i$ ($f_i=[q_i^T,\phi_i^T]^T$). The quality features, attempt to capture the
representativeness for a single sentence. Similarity features $q_i$ for sentence
$s_i$ as we construct the list $L_t$ measure a notion of distance of a proposed sentence to sentences already included in the set.\footnote{
A variety of similarity features were considered, with the simplest being average squared distance of tf-idf vectors. Performance was very stable across different features.
The experiments presented use three types: 1) following the idea in \cite{kulesza2012learning} of similarity
as a volume metric, we compute the
squared volume of the parallelopiped spanned by the TF-IDF vectors of
sentences in the set $L_{t,k} \cup {s_i}$; 
2) the product between $\mbox{det}(G_{L_{t,k} \cup {s_i}})$ and the quality features; 3) the minimum absolute distance of quality features between $s_i$ and each element in $L_{t,k}$.
}

  
Table \ref{DocSumTable} shows the performance (Rouge unigram statistics) comparing SCP with existing algorithms. 
We observe that SCP outperforms existing state-of-the-art approaches, which we denote SubMod \cite{lin2010multi} and DPP \cite{kulesza2012learning}.
``Greedy (Oracle)'' corresponds to the clairvoyant oracle that directly 
optimizes the test Rouge score and thus serves as an upper bound on this class of techniques.
Figure \ref{all_figs} (right) plots Rouge-1R performance as a function of the size of training data, suggesting SCP's superior data-efficiency compared to ConSeqOpt. 

\begin{table}
    \tiny
\centering
  \begin{tabular}{|l || c | c | c | }
    \hline
    \bf{System} & \bf{ROUGE-1F} & \bf{ROUGE-1P} & \bf{ROUGE-1R} \\ 
    \hline 
    SubMod & 37.39 & 36.86 & 37.99 \\ 
    \hline
    DPP  & 38.27 & 37.87 & 38.71 \\
    \hline
    ConSeqOpt & $ 39.02\pm0.07$ & 39.08$\pm 0.07$ & 39.00$\pm0.12$ \\
    \hline
    SCP  & \textbf{39.15$\pm0.15$} & \textbf{39.16$\pm 0.15$} & \textbf{39.17$\pm0.15$} \\
    \hline
    Greedy (Oracle) & 44.92 & 45.14 & 45.24 \\
    \hline
  \end{tabular}
\caption{ROUGE unigram score on the DUC 2004 test set}
  \label{DocSumTable}
\end{table}
\normalsize

%

%% file: acknowledgement.tex
\subsection*{Acknowledgements}
\vspace{-0.06in}
\begin{small}
This research was supported in part by NSF NRI \emph{Purposeful Prediction} project 
and ONR MURIs \emph{Decentralized Reasoning in Reduced Information Spaces} and \emph{Provably Stable Vision-Based Control}.  Yisong Yue was also supported in part by ONR (PECASE) N000141010672 and ONR Young Investigator Program N00014-08-1-0752. We gratefully thank
Martial Hebert for valuable discussions and support.
\end{small}